\documentclass{article} % For LaTeX2e
\usepackage{iclr2027_conference,times}

\usepackage{amsmath,amsfonts,bm}

\def\eqref#1{equation~\ref{#1}}
\def\1{\bm{1}}

\DeclareMathAlphabet{\mathsfit}{\encodingdefault}{\sfdefault}{m}{sl}
\SetMathAlphabet{\mathsfit}{bold}{\encodingdefault}{\sfdefault}{bx}{n}

\usepackage{hyperref}
\hypersetup{hidelinks}
\usepackage{url}
\usepackage{amsmath}
\usepackage{amssymb}
\usepackage{amsthm}
\usepackage{mathtools}
\usepackage{booktabs}
\usepackage{graphicx}
\usepackage{placeins}
\usepackage{float}

\newtheorem{proposition}{Proposition}
\newtheorem{corollary}{Corollary}
\newtheorem{remark}{Remark}

\newcommand{\TV}{\operatorname{TV}}

\title{SAKI: Maximal-Coupling-Routed Teacher Supervision for On-Policy Distillation}

\author{
\textbf{Miteto Wei}$^{1}$ \quad
\textbf{Xiaohan Wang}$^{1,*}$ \quad
\textbf{Zehao Chen}$^{1}$ \quad
\textbf{Jiajun Chai}$^{1}$ \quad
\textbf{Sichao Liu}$^{2}$ \\
\textbf{Li Wang}$^{1}$ \quad
\textbf{Haoyuan Xu}$^{1}$ \quad
\textbf{Zhaoyu Hu}$^{1}$ \quad
\textbf{Wei Lin}$^{1}$ \quad
\textbf{Guojun Yin}$^{1,*}$ \\[0.35em]
{\normalfont $^{1}$Meituan \qquad
$^{2}$KTH Royal Institute of Technology} \\[-0.05em]
{\normalfont\small
\texttt{vizzlin@foxmail.com,wangxiaohan17@meituan.com,yinguojun02@meituan.com}
}
}

\iclrfinalcopy

\begin{document}

\maketitle
\fancyhead[L]{Preprint}
\begingroup
\renewcommand{\thefootnote}{*}
\footnotetext{Corresponding author}
\endgroup
% ==================================================================
% ABSTRACT
% ==================================================================

\begin{abstract}
On-policy distillation (OPD) reduces train--test state mismatch by training a
student on its own generated trajectories.
However, a weak student may initially visit poor, teacher-misaligned prefixes,
forcing the teacher to provide supervision on states that it would rarely
generate under its own policy.
Teacher-guided rollout policies improve the visited state distribution, but
typically retain the same per-prefix reverse-KL objective.
We introduce SAKI (Supervision Allocation with KL-constrained Interpolation),
which \textbf{routes token-level supervision using realized accept/correction
events from maximal coupling}.
We construct a geometrically interpolated behavior policy inside a
student-centered KL trust region and realize it through maximal coupling with
the student.
The coupling preserves a student proposal whenever possible and exposes a
correction event precisely when realizing the guided policy requires an
intervention.
Accepted positions retain reverse-KL supervision, whereas correction
positions switch to direct supervision on the teacher's highest-probability
token.
Because the \textbf{correction probability is exactly $\TV(p_t,q_t)$}, the same
trust-region radius constrains rollout deviation and upper-bounds intervention
and specialized-supervision frequency.
An engine-resident speculative verifier preserves the exact-$q$ trajectory
distribution and coupling semantics while improving matched-workload rollout
throughput by \textbf{$4.22\times$}.
Across seven mathematical reasoning benchmarks, our method
\textbf{improves the matched teacher-guided baseline in Mean@8 and Pass@8
for both 1.7B and 0.6B students}.
Placement controls show that correction-triggered routing outperforms both
equal-budget random and TV-weighted placement.
Fixed-prefix analysis further shows persistent teacher alignment, with larger
gains over random placement at higher initial student--teacher disagreement.
\par\medskip
\noindent\textbf{Code:}
\href{https://github.com/Miteto-sudo/SAKI}
{github.com/Miteto-sudo/SAKI}

\end{abstract}

% ==================================================================
% 1. INTRODUCTION
% ==================================================================

\section{Introduction}
\label{sec:introduction}

Knowledge distillation (KD) transfers capabilities from a strong teacher model
to a smaller student by matching its predictive behavior
~~\citep{bucila2006model,ba2014deep,hinton2015distilling,gou2021survey},
with sequence-level distillation extending this idea to autoregressive
generation ~\citep{kim2016sequence}.
For autoregressive language models, however, conventional distillation suffers
from a state-distribution mismatch:
the student is often trained on fixed or teacher-generated prefixes, while at
inference time it must condition on its own generations
~\citep{bengio2015scheduled,ross2011dagger}.
On-policy distillation (OPD) alleviates this mismatch by letting the student
generate its own trajectories and querying the teacher on states actually
visited by the student
~\citep{gu2024minillm,agarwal2024gkd}.
This is particularly attractive for reasoning distillation, where an early
generation decision can alter the entire subsequent trajectory and hence the
states on which later supervision is provided~\citep{hu2026ziprl}.

Yet OPD's defining strength also creates an important limitation.
\emph{The effectiveness of distillation depends not only on what supervision
the teacher provides, but also on the states at which that supervision is
queried.}
When the student is substantially weaker than the teacher, early mistakes can
compound and drive its rollout toward prefixes that have low probability under
the teacher's own behavior.
Although the teacher distribution remains well-defined at such prefixes, the
teacher is forced to continue from states that it would rarely generate under
its own policy.
The resulting conditional signal may therefore be less representative of the
reasoning behavior that makes the teacher strong and less directly useful for
capability transfer.

This suggests that improving OPD requires considering both
\emph{where} teacher supervision is queried and
\emph{how} the student is updated at the resulting positions~\citep{wang2026self}.

A natural approach is to guide the rollout distribution toward the teacher.
Directly replacing student rollouts with teacher trajectories, however,
introduces the opposite distribution shift:
the resulting states may be too far from the student's current behavior to
form an appropriate online learning distribution.
Trust-Region Behavior Blending (TRB)
~\citep{plyusov2026trb}
addresses this trade-off by constructing, at every prefix, an intermediate
behavior distribution that moves toward the teacher while remaining inside a
student-centered KL trust region.
Given student distribution $p_t$ and teacher distribution $T_t$, TRB
constructs a teacher-guided behavior distribution $q_t$ that moves toward
$T_t$ while satisfying a student-centered constraint
$D_{\mathrm{KL}}(q_t\Vert p_t)\leq\epsilon$.
TRB therefore addresses the \emph{behavior-side} question of where
supervision is queried while retaining the same per-prefix objective.
This leaves a complementary \emph{objective-side} question:
should positions where the guided behavior retains a student proposal receive
the same supervision as positions where it must override that proposal?

Reverse KL is student-weighted: tokens contribute in proportion to the
student distribution ~\citep{lin2026rest}.
Consequently, teacher-preferred tokens with low student support receive weak gradient updates under RKL, highlighting the need for targeted direct supervision at high-conflict positions.

Sampling $q_t$ through maximal coupling provides this distinction
automatically. At each position, the student proposes from $p_t$; the coupling
either retains the proposal or draws the residual correction required to
realize $q_t$. We reuse this realized accept/correction event as an endogenous
supervision router. Under maximal coupling,
$\Pr(C_t=1)=\TV(p_t,q_t)$, which is the minimum disagreement probability
among couplings of $p_t$ and $q_t$; the trust-region constraint further gives
$\Pr(C_t=1)\leq\sqrt{\epsilon/2}$.
Section~\ref{sec:coupling} formalizes these properties. We designate this framework SAKI (Supervision Allocation with KL-constrained Interpolation). 
Under this design, KL-constrained interpolation defines the target intermediate behavior 
distribution, while realized maximal-coupling events serve as an endogenous mechanism 
that adaptively routes token-level supervision.

This creates a natural distinction between two types of positions.
At an \emph{accepted position}, the student proposes behavior that remains
compatible with the teacher-guided rollout.
At a \emph{correction position}, the student's proposed behavior cannot be
retained and the intermediate policy must modify the trajectory.
We hypothesize that these two regimes need not receive identical supervision.

We therefore retain the ordinary sampled-token reverse-KL (RKL) signal at
accepted positions, while switching correction positions to direct
supervision on the teacher's highest-probability token.
Specifically, once maximal coupling identifies a correction, we query the
teacher at the same prefix and optimize the negative log-likelihood of its
Top-1 token.
This exposes the student to the teacher-preferred mode precisely at positions
where the guided rollout cannot retain the student proposal.
Because corrections occur with probability $\TV(p_t,q_t)$, the resulting
teacher supervision is sparse and follows a stochastic, conflict-adaptive
routing pattern determined by the rollout coupling itself rather than by an
external token-selection heuristic.

The residual correction token determines the subsequent rollout prefix,
whereas the teacher Top-1 token determines the local parameter update.
Thus, the same coupling process controls both trajectory construction and
supervision routing.

Exact teacher-guided rollout requires online teacher distributions.
We therefore use engine-resident speculative block verification to amortize
teacher inference while preserving the exact-$q$ trajectory distribution and
the coupling semantics required by our objective.

Our contributions are summarized as follows:
\begin{itemize}

\item
\textbf{Coupling-routed teacher-mode supervision.}
We reuse realized correction events as an endogenous supervision router:
accepted positions retain sampled-token RKL, while corrections receive
direct supervision on the teacher's Top-1 token.
Placement controls show gains over both count-matched random and
TV-weighted conflict-aware placement, while fixed-prefix analysis shows
persistent teacher support and larger gains over random placement at
higher student--teacher conflict.

\item
\textbf{Minimal-intervention teacher-guided rollout.}
We realize the TRB behavior policy through maximal coupling.
The correction probability is exactly
$\TV(p_t,q_t)$, the minimum possible intervention probability among
couplings with these marginals, and satisfies
$\Pr(C_t=1)\leq\sqrt{\epsilon/2}$ under the trust region.

\item
\textbf{Engine-resident exact-$q$ rollout.}
We implement maximal-coupling rollout with engine-resident speculative
block verification, including exact residual correction and
first-rejection commit/rollback.
The system preserves the exact-$q$ trajectory distribution and coupling
semantics while providing a $4.22\times$ matched-workload speedup over
the external-loop implementation.

\end{itemize}

% ==================================================================
% 2. RELATED WORK
% ==================================================================

\section{Related Work}
\label{sec:related_work}

\subsection{On-Policy Distillation and Teacher-Guided Rollouts}

Knowledge distillation has been studied across both strong-to-weak and emerging weak-to-strong settings ~\citep{hinton2015distilling,chen2026weak},
with sequence-level distillation extending the idea to autoregressive
generation
~\citep{kim2016sequence}.
For language models, fixed or teacher-generated trajectories create a mismatch
between prefixes observed during training and those encountered when the
student generates autonomously.
MiniLLM
~\citep{gu2024minillm}
studies reverse-KL distillation for language generation, while Generalized
Knowledge Distillation (GKD)
~\citep{agarwal2024gkd}
provides a framework for training students on their own generated outputs.
These approaches motivate modern OPD, in which teacher supervision is delivered
on states induced by the student's current policy.

On-policy training reduces train--test state mismatch, but it also makes the
training state distribution depend on the current student's quality.
TRB
~\citep{plyusov2026trb}
addresses this issue by replacing pure student rollout with a teacher-guided
behavior distribution constrained by a student-centered KL trust region,
while retaining the same reverse-KL objective at visited prefixes.
We adopt the same intermediate-policy construction, but additionally exploit
the accept/correction information exposed when that policy is realized through
maximal coupling.
Recent work has also explored fine-grained distillation objectives and token-level supervision, including selective KL objectives and reflective credit assignment ~\citep{xing2026tropd,wei2026amr}.
Our focus is on whether the guided rollout process itself can provide the
routing signal for such specialized supervision.

\subsection{Speculative Decoding and Distillation}

Speculative decoding accelerates autoregressive generation by allowing a draft
model to propose multiple tokens that are verified in parallel by a stronger
target model
~\citep{stern2018blockwise,leviathan2023speculative,chen2023accelerating}.
Accepted prefixes can be committed in blocks, reducing the number of expensive
serial target-model decoding steps while preserving the desired sampling
distribution.

Draft-OPD
~\citep{lei2026draftopd}
connects speculative verification with on-policy distillation and demonstrates
that verification outcomes can provide useful training structure in addition
to computational acceleration.
Our use of speculative execution serves a different role:
the student is the capability-distilled model itself, and speculative
verification is used to efficiently realize a separate teacher-guided
intermediate distribution $q$.
The resulting maximal-coupling decisions are then reused to route token-level
supervision.
Thus, speculative execution in our framework simultaneously amortizes
teacher inference and exposes the coupling events that connect trajectory
construction with the distillation objective.

% ==================================================================
% 3. METHOD
% ==================================================================

\section{Method}
\label{sec:method}

Figure~\ref{fig:method_overview} summarizes one rollout-and-update step.
At prefix $h_t$, the rollout student $p_t$ and teacher $T_t$ define the
trust-region behavior distribution $q_t$. Maximal coupling then realizes
$q_t$ while exposing an accept/correction event that determines the rollout
token and routes the local distillation objective.

\begin{figure}[t]
  \centering
  \includegraphics[width=\linewidth]{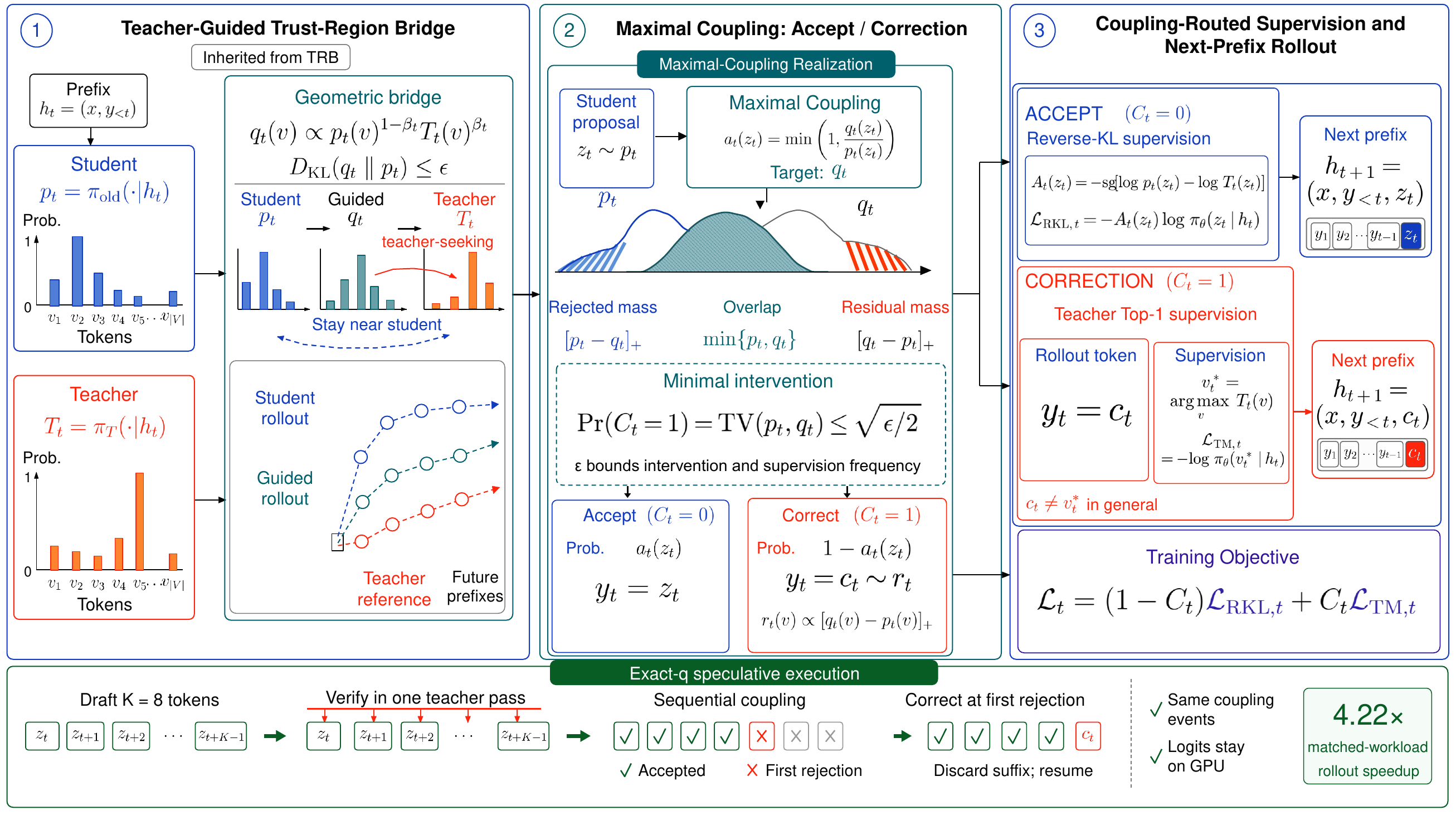}
    \caption{
    \textbf{Overview of SAKI and its exact-\(q\) execution.}
    The TRB bridge defines $q_t$; maximal coupling either retains a student
    proposal with RKL supervision or draws a residual rollout correction and
    routes teacher-Top-1 supervision. The residual token controls the next prefix,
    whereas the teacher mode controls the update. Bottom: an engine-resident
    $K$-token verifier commits proposals only through the first rejection,
    commits an exact residual correction there, and discards the invalid suffix.
}
  \label{fig:method_overview}
\end{figure}

\subsection{Problem Setup}

Let $h_t=(x,y_{<t})$ denote the prefix at decoding step $t$.
We write $p_t(v)=\pi_{\mathrm{old}}(v\mid h_t)$ for the frozen student
behavior distribution and $T_t(v)=\pi_T(v\mid h_t)$ for the fixed teacher
distribution.
The trainable student is denoted by $\pi_\theta$.

Standard student-rollout OPD samples $y_t\sim p_t$ and uses the detached
log-ratio advantage
\[
A_t
=
-\operatorname{sg}\!\left[\log p_t(y_t)-\log T_t(y_t)\right].
\]
Under exact student on-policy sampling, this gives the standard one-sample
score-function estimator associated with reverse KL
~\citep{williams1992reinforce};
Appendix~\ref{app:k1} provides the derivation.
In our teacher-guided rollout, we reuse the same sampled log-ratio
signal on student proposals retained by maximal coupling.
As shown in Section~\ref{sec:objective}, the resulting RKL
supervision acts on the student--guided overlap mass
$\min\{p_t,q_t\}$, while correction positions receive the
teacher-mode objective introduced below.

\subsection{Teacher-Guided Trust-Region Rollout}
\label{sec:bridge}

Following TRB~~\citep{plyusov2026trb}, we construct the
teacher-guided behavior policy by geometric interpolation
between the student and teacher:
\begin{equation}
q_{\beta,t}(v)
=
\frac{
p_t(v)^{1-\beta_t}
T_t(v)^{\beta_t}
}{
Z_t(\beta_t)
},
\qquad
Z_t(\beta_t)
=
\sum_u
p_t(u)^{1-\beta_t}
T_t(u)^{\beta_t}.
\label{eq:bridge}
\end{equation}
At every position, we choose the largest feasible $\beta_t\in[0,1]$ satisfying
\begin{equation}
D_{\mathrm{KL}}(q_{\beta,t}\Vert p_t)
\leq
\epsilon.
\label{eq:trust_region}
\end{equation}
Thus,
$\epsilon=0$ recovers $q_t=p_t$,
while a sufficiently large trust region permits $q_t=T_t$.
For completeness, Appendix~\ref{app:trb} outlines the derivation of this geometric bridge following TRB ~\citep{plyusov2026trb}.

\subsection{Maximal Coupling and Correction Events}
\label{sec:coupling}

To sample exactly from $q_t$ while exposing the relationship between student
and guided behavior, we use maximal coupling
~\citep{lindvall2002coupling}.

The student first proposes
\begin{equation}
z_t\sim p_t.
\end{equation}
The proposal is accepted with probability
\begin{equation}
a_t(z_t)
=
\min
\left(
1,
\frac{q_t(z_t)}{p_t(z_t)}
\right).
\label{eq:accept}
\end{equation}
If accepted,
\begin{equation}
y_t=z_t,
\qquad
C_t=0.
\end{equation}
Otherwise, we draw a correction token from the positive residual distribution
\begin{equation}
r_t(v)
=
\frac{
[q_t(v)-p_t(v)]_+
}{
\sum_u [q_t(u)-p_t(u)]_+
},
\label{eq:residual}
\end{equation}
and set
\begin{equation}
y_t=c_t\sim r_t,
\qquad
C_t=1.
\end{equation}

The following proposition gives the key distributional semantics of the
correction event.

\begin{proposition}[Distributional semantics of coupling corrections]
\label{prop:coupling}
For maximal coupling between distributions $p$ and $q$, the final output has
marginal distribution $q$, and
\begin{equation}
\Pr(C=1)
=
\TV(p,q)
=
1-\sum_v\min\{p(v),q(v)\}.
\label{eq:correction_tv}
\end{equation}
\end{proposition}

The conditional rejected-proposal and residual-correction distributions are
derived in Appendix~\ref{app:maximal_coupling}.

\begin{proposition}[Minimal-intervention property]
\label{prop:minimal_intervention}
Let $\Gamma(p,q)$ denote the set of all couplings with marginals $p$ and $q$.
For any $\gamma\in\Gamma(p,q)$,
\begin{equation}
\Pr_{\gamma}(Z\neq Y)\geq\TV(p,q),
\end{equation}
with equality under maximal coupling.
Thus, maximal coupling realizes the guided marginal with the minimum possible
intervention probability.
\end{proposition}

\begin{corollary}[Trust-region control of intervention rate]
\label{cor:intervention_rate}
If the guided distribution satisfies
\begin{equation}
D_{\mathrm{KL}}(q_t\Vert p_t)\leq\epsilon,
\end{equation}
then the maximal-coupling correction probability satisfies
\begin{equation}
\Pr(C_t=1)
=
\TV(p_t,q_t)
\leq
\sqrt{
\frac{1}{2}
D_{\mathrm{KL}}(q_t\Vert p_t)
}
\leq
\sqrt{\frac{\epsilon}{2}}.
\label{eq:intervention_rate_bound}
\end{equation}
\end{corollary}

Thus, $\epsilon$ controls not only the distributional distance of the rollout
from the student but also the maximum local probability that the rollout must
override a student proposal.

Appendix~\ref{app:correction_activity} reports the empirical correction
trajectory under the annealed trust-region schedule: the observed correction probability remains below the Pinsker upper bound and
vanishes when $\epsilon$ reaches zero.

The proof is given in Appendix~\ref{app:maximal_coupling}.
Proposition~\ref{prop:coupling} shows that a correction transfers the trajectory
from probability mass overrepresented by the student relative to $q$ toward
mass underrepresented by the student relative to $q$.
Hence, correction is not an arbitrary gating heuristic:
it has an explicit distributional meaning.

Unlike scalar conflict measures such as TV or KL, $C_t$ is a realized, proposal-dependent event: it identifies positions where the sampled student proposal cannot be retained under the minimum-intervention coupling that exactly realizes $q_t$.

For the geometric bridge, the positive residual is supported only on tokens
whose teacher-to-student likelihood ratio exceeds a prefix-dependent threshold;
Appendix~\ref{app:bridge_interpretation} gives the derivation.

\subsection{Coupling-Routed Supervision}
\label{sec:objective}

SAKI allocates supervision using the realized coupling
indicator $C_t$.
Accepted positions retain the sampled-token RKL update,
whereas correction positions receive teacher-mode supervision.

For a token $v$ at prefix $h_t$, define the sampled-token RKL loss
\begin{equation}
L_{\mathrm{RKL},t}(v)
=
-
A_t(v)
\log\pi_\theta(v\mid h_t),
\label{eq:k1_loss_token}
\end{equation}
where
\begin{equation}
A_t(v)
=
-
\operatorname{sg}
\left[
\log p_t(v)-\log T_t(v)
\right].
\label{eq:k1_advantage_token}
\end{equation}
At an accepted position, we evaluate this loss on the retained proposal
$v=z_t=y_t$.
We retain the ordinary RKL signal at accepted positions and replace
it with teacher-mode supervision when $C_t=1$.

\paragraph{Teacher-mode supervision.}
At a correction position, let
\begin{equation}
v_t^*
=
\operatorname*{arg\,max}_v T_t(v)
\end{equation}
denote the teacher's highest-probability token at the same prefix.
We replace the RKL term with
\begin{equation}
L_{\mathrm{TM},t}
=
-\log \pi_\theta(v_t^*\mid h_t).
\label{eq:teacher_mode_loss}
\end{equation}
Thus, correction detection and correction supervision play decoupled roles:
maximal coupling determines \emph{where} specialized supervision is activated,
while the teacher mode determines \emph{what} the student learns.
We adopt the deterministic teacher mode ($v_t^* = \operatorname*{arg\,max}_v T_t(v)$) rather than stochastic teacher sampling ($\tilde{v} \sim T_t$), as the mode provides a sharp, low-variance target precisely at high-conflict positions.
Empirically, teacher-mode supervision outperforms stochastic teacher sampling by $+0.62$ Mean@8 and $+1.56$ Pass@8 on the 1.7B student (Appendix~\ref{app:correction_supervision}, Table~\ref{tab:correction_supervision}), confirming the benefit of mode-seeking updates under coupling-detected divergence.

This design naturally decouples exploration steering from parameter optimization: 
the residual token $c_t$ steers the autoregressive prefix along the guided target 
distribution $q_t$, whereas the teacher mode $v_t^*$ provides a low-variance, 
deterministic anchor for parameter updates at divergence points.

Let $M_t\in\{0,1\}$ denote the valid response-token mask and let
$N_{\mathrm{valid}}=\sum_t M_t$.
Our coupling-aware objective is
\begin{equation}
L
=
\frac{1}{N_{\mathrm{valid}}}
\left[
\sum_t
M_t(1-C_t)L_{\mathrm{RKL},t}
+
\sum_t
M_tC_tL_{\mathrm{TM},t}
\right].
\label{eq:general_objective}
\end{equation}

At a fixed prefix, suppressing the time index for clarity,
Proposition~\ref{prop:coupling} gives
\begin{equation}
\mathbb{E}
\left[
(1-C)L_{\mathrm{RKL}}(Z)
\right]
=
\sum_v
\min\{p(v),q(v)\}
L_{\mathrm{RKL}}(v).
\label{eq:expected_overlap_k1}
\end{equation}
Thus, accepted-position RKL supervision acts exactly on the
student--guided overlap mass $\min\{p,q\}$.
Correction positions form the complementary branch of the same
maximal-coupling realization: at these positions, the sampled student
proposal cannot be retained while exactly realizing the guided rollout
distribution $q$.
We reuse this realized branch to switch from RKL to teacher-mode
supervision, without introducing an additional token-selection rule or
routing threshold.

In our main training schedule, this correction supervision is transient:
we anneal $\epsilon$ to zero so that the method eventually returns exactly
to student-rollout RKL training.
Section~\ref{sec:mechanism} analyzes the distributional effect of this
transition.

\subsection{Efficient Engine-Resident Exact-$q$ Rollout}
\label{sec:speculative}

Token-wise exact coupling requires both student and teacher distributions at
every generated position and is therefore substantially more expensive than
student-only OPD. We amortize this online computation with speculative block
verification. At each wave, the frozen rollout student drafts up to $K$
tokens from the current true prefix, and the teacher evaluates the
corresponding proposal prefixes in one batched verification. Inside the
inference engine, we construct $q_t$ for all block positions, solve the
trust-region coefficients $\beta_t$ in batch, and evaluate maximal-coupling
acceptance from left to right.

The engine commits the longest consecutively accepted proposal prefix. If the
first rejection occurs at position $j$, it samples
$c_j\propto[q_j-p_j]_+$, commits that correction, discards all later
speculative tokens and KV states, and resumes from the corrected prefix.
Student and teacher full-vocabulary logits remain on GPU; the training loop
receives only compact token-aligned metadata, including committed tokens,
the correction mask, selected log-probabilities, and teacher-mode targets.
Because proposal prefixes coincide with true prefixes up to the first
rejection and no post-correction speculative state is reused, this execution
has the same autoregressive law and coupling events as sequential exact-$q$
sampling. Appendix~\ref{app:spec_exact} proves exactness, and
Appendix~\ref{app:systems} gives implementation and validation details.

% ==================================================================
% 4. EXPERIMENTS
% ==================================================================

\section{Experiments}
\label{sec:experiments}

\subsection{Experimental Setup}

\paragraph{Models and shared protocol.}
We distill Qwen3-0.6B-Base and Qwen3-1.7B-Base students
~\citep{yang2025qwen3} from the same Qwen3-4B-Base-GRPO teacher on
DAPO-Math-17K~~\citep{yu2025dapo}. Within each student scale, all
methods use the same student and teacher checkpoints, training prompts,
200-step budget, rollout batch of 64 prompts with eight responses per prompt,
and student-only evaluation protocol. Complete optimization, sampling, and
model-compatibility details are in Appendix~\ref{app:training_details}.

\paragraph{Baselines.}
Where an RKL term is present, OPD, ExOPD, TRB, Random-TM,
TV-Weighted-TM, and Ours use the same sampled-token log-ratio update.
ExOPD~~\citep{yang2026gopd} follows the official G-OPD code with
$\lambda=1.25$,
\texttt{only\_reverse\_kl\_advantages=True}, and the initial student as the
fixed reference.
SKD~~\citep{xu2025speculative} is reproduced using its canonical Top-25 speculative
rollout and full-distribution KL objective under our unified training and
evaluation protocol.

Random-TM and TV-Weighted-TM are placement controls for the 1.7B student.
Both match the number of teacher-mode updates used by Ours on each
exact-$q$ trajectory. Random-TM uses count-matched random placement,
whereas TV-Weighted-TM samples positions without replacement according
to the local $\TV(p_t,q_t)$ score. For TV-Weighted-TM, the realized
correction mask determines only the per-trajectory supervision budget,
not the selected locations.

TRB, Random-TM, TV-Weighted-TM, and Ours share an exact-$q$ rollout with
$K=8$ and $\epsilon:0.02\rightarrow0$ linearly over the first 50 steps,
directly adopting the established schedule from
\citet{plyusov2026trb}.
To strictly isolate the algorithmic impact of supervision routing from confounding 
trajectory dynamics, all exact-$q$ variants share the identical rollout schedule 
established in prior work ~\citep{plyusov2026trb}. This strictly controlled protocol 
ensures that all empirical improvements are directly attributable to token-level 
supervision allocation rather than behavioral tuning artifacts.

\paragraph{Evaluation.}
We evaluate MATH-500~~\citep{hendrycks2021math,lightman2023lets}, HMMT-Feb26, AIME 2026,
AIME 2025, AMC 2023, Minerva Math~~\citep{lewkowycz2022minerva}, and
OlympiadBench~~\citep{he2024olympiadbench}. We sample eight responses per problem at temperature
1.0 with maximum response length 16,384. Mean@8 is the average correctness over
the eight samples, Pass@8 is the fraction of problems with at least one
correct sample, and Average is the unweighted mean over the seven displayed
benchmarks. Full decoding settings are in
Appendix~\ref{app:training_details}.

\subsection{Main Results}
\begin{table*}[t]
\centering
\renewcommand{\arraystretch}{0.88}
\caption{
\textbf{Mean@8 accuracy} (\%). Average is the mean over the seven
benchmarks; \textbf{bold} and \textbf{underline} denote the best and second-best trained method
within each student size, respectively.
}
\label{tab:opd-mean8}
\resizebox{\textwidth}{!}{%
\begin{tabular}{lrrrrrrrr}
\toprule
Method & MATH-500 & HMMT & AIME26 & AIME25 & AMC23 & Minerva & Olympiad & Average \\
\midrule
\textit{Teacher (4B)}
    & 84.1 & 15.2 & 19.2 & 19.6 & 60.9 & 38.0 & 51.9 & 41.3 \\
\midrule

Student (1.7B)
    & 10.2 & 0.0 & 0.4 & 0.4 & 3.8 & 3.1 & 3.2 & 3.0 \\
OPD
    & 70.7 & 2.3 & 6.3 & 8.3 & 42.8 & 27.2 & 33.7 & 27.3 \\
ExOPD
    & 69.5 & 2.7 & 5.4 & 7.9 & 44.4 & 27.5 & 32.1 & 27.1 \\
SKD
    & 64.3 & 3.4 & 5.4 & 4.6 & 40.9 & 23.9 & 29.1 & 24.5 \\
TRB
    & 70.7 & 4.9 & 6.3 & 8.8 & 43.1 & 26.5 & 34.7 & \underline{27.9} \\
\textbf{SAKI (Ours)}
    & 70.9 & 5.7 & 7.9 & 8.3 & 47.2 & 28.1 & 34.9 & \textbf{29.0} \\
\midrule

Student (0.6B)
    & 6.2 & 0.4 & 0.0 & 0.0 & 0.9 & 2.2 & 2.1 & 1.7 \\
OPD
    & 52.7 & 0.8 & 0.4 & 2.5 & 26.9 & 15.4 & 20.7 & 17.1 \\
ExOPD
    & 52.1 & 1.5 & 0.8 & 1.7 & 31.3 & 15.9 & 20.8 & \underline{17.7} \\
SKD
    & 41.7 & 1.1 & 0.4 & 0.4 & 19.1 & 11.0 & 13.3 & 12.4 \\
TRB
    & 52.0 & 1.5 & 0.4 & 2.5 & 28.4 & 14.9 & 20.9 & 17.2 \\
\textbf{SAKI (Ours)}
    & 52.6 & 1.9 & 1.7 & 4.2 & 31.9 & 15.3 & 21.2 & \textbf{18.4} \\
\bottomrule
\end{tabular}%
}
\end{table*}

\begin{table*}[t]
\centering
\renewcommand{\arraystretch}{0.88}
\caption{
\textbf{Pass@8 accuracy} (\%). Average is the mean over the seven
benchmarks; \textbf{bold} and \textbf{underline} denote the best and second-best trained method
within each student size, respectively.
}
\label{tab:opd-pass8}
\resizebox{\textwidth}{!}{%
\begin{tabular}{lrrrrrrrr}
\toprule
Method & MATH-500 & HMMT & AIME26 & AIME25 & AMC23 & Minerva & Olympiad & Average \\
\midrule
\textit{Teacher (4B)}
    & 94.4 & 24.2 & 30.0 & 26.7 & 97.5 & 57.4 & 70.2 & 57.2 \\
\midrule

Student (1.7B)
    & 49.6 & 0.0 & 3.3 & 3.3 & 20.0 & 19.9 & 18.4 & 16.4 \\
OPD
    & 90.0 & 9.1 & 20.0 & 20.0 & 72.5 & 46.3 & 57.2 & \underline{45.0} \\
ExOPD
    & 87.6 & 6.1 & 20.0 & 20.0 & 75.0 & 45.6 & 56.7 & 44.4 \\
SKD
    & 86.6 & 6.1 & 16.7 & 13.3 & 67.5 & 44.5 & 53.6 & 41.2 \\
TRB
    & 90.4 & 9.1 & 16.7 & 23.3 & 70.0 & 44.9 & 57.9 & 44.6 \\
\textbf{SAKI (Ours)}
    & 90.0 & 18.2 & 20.0 & 23.3 & 75.0 & 47.4 & 58.5 & \textbf{47.5} \\
\midrule

Student (0.6B)
    & 35.8 & 3.0 & 0.0 & 0.0 & 7.5 & 13.2 & 12.9 & 10.3 \\
OPD
    & 78.2 & 3.0 & 3.3 & 6.7 & 52.5 & 34.2 & 43.6 & 31.6 \\
ExOPD
    & 77.2 & 6.1 & 6.7 & 6.7 & 57.5 & 36.0 & 41.9 & 33.2 \\
SKD
    & 71.0 & 6.1 & 3.3 & 3.3 & 57.5 & 31.3 & 34.4 & 29.6 \\
TRB
    & 77.4 & 9.1 & 3.3 & 6.7 & 60.0 & 34.9 & 43.9 & \underline{33.6} \\
\textbf{SAKI (Ours)}
    & 81.2 & 9.1 & 10.0 & 13.3 & 60.0 & 32.0 & 43.3 & \textbf{35.6} \\
\bottomrule
\end{tabular}%
}
\end{table*}

\paragraph{Performance summary.}
Tables~\ref{tab:opd-mean8} and~\ref{tab:opd-pass8} show that Ours achieves
the best macro Mean@8 and Pass@8 at both student scales. Relative to TRB, the
1.7B student improves from 27.9/44.6 to 29.0/47.5
($+1.1/+2.9$ points), and the 0.6B student improves from 17.2/33.6 to
18.4/35.6 ($+1.2/+2.0$ points). Relative to SKD, the gains are
$+4.5/+6.3$ points for 1.7B and $+6.0/+6.0$ points for 0.6B.
Mean@8 improves over TRB in 13 of 14 student--benchmark pairs.

\paragraph{Placement controls.}
For the 1.7B student, Random-TM reaches 28.30 Mean@8 and 45.50 Pass@8.
TV-Weighted-TM matches the same per-trajectory teacher-mode budget but
weights placement by the local $\TV(p_t,q_t)$ score, reaching
28.24 Mean@8 and 46.77 Pass@8.
Ours reaches 29.00/47.50, exceeding TV-Weighted-TM by
$+0.76$ Mean@8 and $+0.73$ Pass@8.
Thus, scalar conflict-aware placement is useful, while realized coupling
corrections provide a stronger routing signal under the matched setup.
Aggregate values are reported in
Table~\ref{tab:placement_control};
Figure~\ref{fig:main_results_summary}(c) visualizes the Random-TM
decomposition.

\subsection{Mechanistic Analysis:
Persistent Alignment and Conflict-Adaptive Teacher Support}
\label{sec:mechanism}

\paragraph{Matched fixed-prefix probe.}
We compare matched TRB, Random-TM, and Ours runs on a frozen probe of
2,048 positions from 213 prompts. We track student probability on the
teacher Top-1 token ($C_1$) and Top-16 set ($C_{16}$), and stratify
positions by initial $D_{\mathrm{KL}}(T\Vert S)$. Differences are
prompt-equal, and confidence intervals use prompt-clustered bootstrap.
Correction-triggered supervision ends at step 51; full probe construction
and statistics are in Appendix~\ref{app:mechanism_details}.

\begin{figure*}[t]
\centering
\includegraphics[width=0.98\textwidth]
{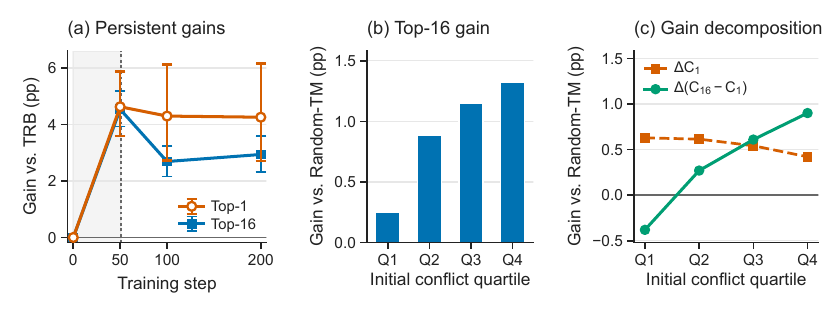}
\caption{
\textbf{Persistent and conflict-adaptive teacher support.}
$C_1$ is student probability on the teacher Top-1 token, and $C_{16}$ is
student probability mass on the teacher Top-16 set.
(a) Gains over TRB during training; shading marks the correction-active phase
and the dotted line marks step 51. Error bars are 95\% prompt-clustered
bootstrap CIs.
(b,c) Step-200 gains over Random-TM by initial
$D_{\mathrm{KL}}(T\Vert S)$ quartile (Q1 lowest, Q4 highest).
The Q4--Q1 contrasts are $+1.07$ pp for $\Delta C_{16}$ and
$+1.28$ pp for $\Delta(C_{16}-C_1)$.
}
\label{fig:mechanism_alignment}
\end{figure*}

\paragraph{Teacher-mode supervision establishes persistent alignment.}
Figure~\ref{fig:mechanism_alignment}(a) shows that correction-triggered
teacher-mode supervision rapidly increases teacher-supported probability mass
during the correction-active phase.
At step 50, Ours exceeds TRB by $4.63$ percentage points in $C_1$ and
$4.54$ points in $C_{16}$.
After teacher-mode supervision is disabled at step 51, these gains remain
visible: at step 200, after 149 RKL-only steps, the corresponding advantages are
still $4.26$ and $2.94$ points.
Additional divergence and entropy statistics are reported in
Appendix~\ref{app:mechanism_details}.

\paragraph{Coupling-based placement is conflict adaptive.}
The equal-budget Random-TM control tests whether correction-based placement
provides benefits beyond teacher-mode supervision alone.
Relative to Random-TM, correction-triggered placement increases $C_1$ by
$0.648$ points and $C_{16}$ by $0.816$ points at step 200.
More importantly, Figure~\ref{fig:mechanism_alignment}(b,c) shows that the
$C_{16}$ advantage grows from $+0.25$ points in the lowest-conflict quartile
to $+1.32$ points in the highest-conflict quartile
(Q4--Q1: $+1.07$ points; 95\% CI: [0.87, 1.29]).
This increase is primarily carried by teacher-supported non-argmax tokens
(Q4--Q1: $+1.28$ points), while the corresponding $C_1$ interaction is not
significant.
These results show that maximal-coupling corrections provide an effective conflict-adaptive routing signal beyond the effect of supervision budget alone.

\subsection{Exact-$q$ Rollout Efficiency}
\label{sec:rollout_efficiency}

We measure committed response tokens per generation wall-clock second under
the same 64-prompt $\times$ 8-response, maximum-length 7,168, $K=8$
workload. The engine-resident backend achieves 3,276 tokens/s compared to 776 tokens/s 
for the external-loop implementation—delivering a $4.22\times$ speedup under 
identical workloads. Remarkably, despite executing full online teacher verification 
and batched trust-region interpolation at every wave, the engine sustains over 
42\% of the theoretical throughput ceiling defined by unguided student-only generation 
(Table~\ref{tab:q_engine_throughput}).
Appendix~\ref{app:systems} reports timing breakdowns, model/cache lifecycle,
and correctness checks.

\begin{table}[t]
\centering
\small
\caption{\textbf{Exact-$q$ rollout} throughput under the matched workload above.}
\label{tab:q_engine_throughput}
\begin{tabular}{lrr}
\toprule
Backend & Tok/s & Relative \\
\midrule
Student-only reference       & 7,760 & -- \\
External-loop exact-$q$ K8   &   776 & 1.00$\times$ \\
Engine-resident exact-$q$ K8 & 3,276 & \textbf{4.22$\times$} \\
\bottomrule
\end{tabular}
\end{table}

\section{Conclusion}
\label{sec:conclusion}

We introduce SAKI, which realizes a teacher-guided
trust-region policy through maximal coupling and routes teacher-mode
supervision at correction events. It achieves the best macro Mean@8 and
Pass@8 at both student scales. Placement controls and fixed-prefix
analysis support correction-based, conflict-adaptive routing. An
engine-resident exact-$q$ verifier preserves sampling semantics while
providing a $4.22\times$ matched-workload speedup.

\subsection*{AI use statement}

In this work, generative AI tools were used to assist with checking and
refining mathematical proofs, code development, and language editing of
the manuscript. All mathematical arguments were independently verified
by the authors. AI-assisted code was reviewed and tested by the authors
for correctness, and all AI-assisted textual edits were reviewed before
inclusion in the manuscript. The authors take full responsibility for the
final content of this work, including text, claims, code, and other
artifacts produced with the aid of generative AI.

\subsection*{Reproducibility statement}

We provide supplementary code containing a reference implementation
of the core components of SAKI.
Appendix~\ref{app:training_details} reports the training and evaluation
configuration, Appendices~\ref{app:maximal_coupling} and
\ref{app:spec_exact} provide the coupling and exact-execution derivations,
and Appendix~\ref{app:systems} describes the systems implementation and
validation details.

% ==================================================================
% REFERENCES
% ==================================================================

\bibliography{iclr2027_conference}
\bibliographystyle{iclr2027_conference}

% ==================================================================
% APPENDIX
% ==================================================================

\appendix

% ==================================================================
% A. TRB BRIDGE DERIVATION
% ==================================================================

\section{Teacher Construction, Model Compatibility, and Training Details}
\label{app:training_details}

\paragraph{Teacher construction.}
We initialize the teacher from Qwen3-4B-Base and reproduce the
Qwen3-4B-Base-GRPO construction of
\citet{li2026rethinking}.
We train the teacher with GRPO on the processed DAPO-Math-17K dataset~\citep{yang2026your}.
Following their data format, each question is augmented with the instruction
``Please reason step by step, and put your final answer within
\texttt{\textbackslash boxed\{\}}.''
We use exactly the reported GRPO hyperparameter recipe, summarized in
Table~\ref{tab:teacher_grpo_config}.

\begin{table}[t]
\centering
\small
\caption{
\textbf{Training configuration} used to construct the \textbf{Qwen3-4B-Base-GRPO teacher},
following \citet{li2026rethinking}.
}
\label{tab:teacher_grpo_config}
\begin{tabular}{@{}p{0.42\linewidth}p{0.50\linewidth}@{}}
\toprule
Setting & Value \\
\midrule
Initialization & Qwen3-4B-Base \\
RL algorithm & GRPO \\
Training data & Processed DAPO-Math-17K \\
Training epochs & 1 \\
Prompt batch size & 64 \\
Reported micro-batch size & 64 \\
Responses per prompt & 8 \\
Maximum prompt length & 1,024 \\
Maximum response length & 7,168 \\
Learning rate & $1\times10^{-6}$ \\
Sampling temperature & 1.0 \\
Top-$p$ & 1.0 \\
Repetition penalty & 1.0 \\
KL regularization & None \\
Additional KL coefficient & 0 \\
Loss aggregation & Token mean \\
Thinking template & Disabled \\
\bottomrule
\end{tabular}
\end{table}

The teacher uses the released rule-based mathematical reward function rather
than a learned reward model ~\citep{liu2026cdrrm}.
Format reward and reference-model KL regularization are disabled.

\paragraph{Model compatibility.}
Our rollout and maximal-coupling implementation operate at the exact token
level.
The publicly released Qwen3 Base and post-trained checkpoints use different
default termination-token conventions.
We therefore use Base-to-Base teacher--student pairs so that vocabulary,
termination behavior, and response masking remain naturally aligned.
All prompts are formatted with \texttt{enable\_thinking=False}.

\paragraph{Matched student-training protocol.}
The 0.6B and 1.7B students use the same training protocol; the student
parameter count is the primary experimental difference.
Both models are trained on the same 17,917 DAPO-Math prompts for 200
optimization steps.
Table~\ref{tab:student_training_config} reports the shared configuration.

\begin{table}[t]
\centering
\small
\caption{
\textbf{Shared training configuration} for the 0.6B and 1.7B students.
}
\label{tab:student_training_config}
\begin{tabular}{@{}ll@{}}
\toprule
Setting & Value \\
\midrule
Students
    & Qwen3-0.6B / 1.7B Base \\
Teacher
    & Qwen3-4B-Base-GRPO \\
Training
    & 200 steps; AdamW; LR $1\times10^{-6}$ \\
Weight decay / grad clip
    & $0.01$ / $1.0$ \\
Rollout batch
    & $64$ prompts $\times$ $8$ responses $=512$ \\
Sampling
    & $T=1.0$, top-$p=1.0$, no top-$k$ \\
Maximum lengths
    & 7,168 \\
Loss aggregation
    & Token mean over valid response tokens \\
Guided rollout(Ours)
    & $K=8$, $\epsilon:0.02\rightarrow0$ over 50 steps \\
\bottomrule
\end{tabular}
\end{table}

\section{Derivation of the Trust-Region Geometric Bridge}
\label{app:trb}

This section reproduces the geometric bridge underlying TRB
~\citep{plyusov2026trb}
for completeness.

At a fixed prefix, consider the problem
\begin{equation}
\begin{aligned}
\min_{q}
\quad&
D_{\mathrm{KL}}(q\Vert T)
\\
\text{s.t.}
\quad&
D_{\mathrm{KL}}(q\Vert p)\leq\epsilon,
\\
&
\sum_v q(v)=1.
\end{aligned}
\label{eq:trb_optimization}
\end{equation}
Its Lagrangian is
\begin{equation}
\mathcal{L}(q,\eta,\xi)
=
D_{\mathrm{KL}}(q\Vert T)
+
\eta
\left(
D_{\mathrm{KL}}(q\Vert p)-\epsilon
\right)
+
\xi
\left(
\sum_vq(v)-1
\right),
\end{equation}
where $\eta\geq0$.

Expanding the KL terms gives
\begin{align}
\mathcal{L}
=&
\sum_v
q(v)
\left[
\log q(v)-\log T(v)
\right]
\nonumber\\
&
+
\eta
\sum_v
q(v)
\left[
\log q(v)-\log p(v)
\right]
+
\text{const}.
\end{align}
Taking the derivative with respect to $q(v)$,
\begin{equation}
(1+\eta)\log q(v)
-
\log T(v)
-
\eta\log p(v)
+
\text{const}
=
0.
\end{equation}
Hence,
\begin{equation}
q(v)
\propto
p(v)^{\frac{\eta}{1+\eta}}
T(v)^{\frac{1}{1+\eta}}.
\end{equation}
Defining
\begin{equation}
\beta
=
\frac{1}{1+\eta}
\in(0,1],
\end{equation}
we obtain
\begin{equation}
q_\beta(v)
=
\frac{
p(v)^{1-\beta}
T(v)^\beta
}{
\sum_u
p(u)^{1-\beta}
T(u)^\beta
}.
\end{equation}
If the teacher itself lies inside the trust region, the constraint is inactive,
$\eta=0$, and $\beta=1$, giving $q=T$.
As the trust-region radius approaches zero, $\eta\rightarrow\infty$,
$\beta\rightarrow0$, and $q\rightarrow p$.

% ==================================================================
% B. K1 / RKL
% ==================================================================

\section{Single-Sample Estimation of Reverse KL}
\label{app:k1}

At a fixed prefix, the reverse KL from student $p_\theta$ to teacher $T$ is
\begin{equation}
D_{\mathrm{KL}}(p_\theta\Vert T)
=
\sum_v
p_\theta(v)
\left[
\log p_\theta(v)-\log T(v)
\right].
\label{eq:rkl_appendix}
\end{equation}

For a sample $Y\sim p_\theta$, define
\begin{equation}
k_1(Y)
=
\log p_\theta(Y)-\log T(Y).
\label{eq:k1_value}
\end{equation}
Taking the expectation under the student distribution gives
\begin{align}
\mathbb{E}_{Y\sim p_\theta}[k_1(Y)]
&=
\sum_v
p_\theta(v)
\left[
\log p_\theta(v)-\log T(v)
\right]
\\
&=
D_{\mathrm{KL}}(p_\theta\Vert T).
\label{eq:k1_value_expectation}
\end{align}
Thus, the sampled log-ratio is a one-sample Monte Carlo estimator of the
full-vocabulary reverse-KL value under exact student sampling.

Since the teacher is fixed,
\begin{align}
\nabla_\theta
D_{\mathrm{KL}}(p_\theta\Vert T)
=&
\sum_v
\nabla_\theta p_\theta(v)
\left[
\log p_\theta(v)-\log T(v)
\right]
\nonumber\\
&
+
\sum_v
p_\theta(v)
\nabla_\theta\log p_\theta(v).
\end{align}
Using
\begin{equation}
\nabla_\theta p_\theta(v)
=
p_\theta(v)
\nabla_\theta\log p_\theta(v)
\end{equation}
and the score-function identity
\begin{equation}
\mathbb{E}_{v\sim p_\theta}
\left[
\nabla_\theta\log p_\theta(v)
\right]
=
0,
\end{equation}
we obtain
\begin{equation}
\nabla_\theta
D_{\mathrm{KL}}(p_\theta\Vert T)
=
\mathbb{E}_{v\sim p_\theta}
\left[
\left(
\log p_\theta(v)-\log T(v)
\right)
\nabla_\theta\log p_\theta(v)
\right].
\label{eq:rkl_gradient}
\end{equation}
Therefore, the same sampled log-ratio provides both a Monte Carlo estimator of the
reverse-KL value in Eq.~\ref{eq:k1_value_expectation} and the coefficient of
an unbiased score-function estimator of its gradient under exact student
sampling.

Equivalently, defining a detached advantage
\begin{equation}
A(y)
=
-
\operatorname{sg}
\left[
k_1(y)
\right],
\end{equation}
the policy-gradient loss
\begin{equation}
L_{\mathrm{PG}}
=
-A(y)\log p_\theta(y)
\end{equation}
has gradient
\begin{equation}
\nabla_\theta L_{\mathrm{PG}}
=
k_1(y)
\nabla_\theta\log p_\theta(y),
\end{equation}
matching the one-sample gradient estimator in
Eq.~\ref{eq:rkl_gradient}.

\begin{remark}[Characterization on Guided Rollouts]
Under teacher-guided rollout ($Y \sim q_t$), evaluating the log-ratio update 
exclusively on accepted proposals corresponds to a localized score-function estimator 
supported exactly on the student--guided intersection measure $\min\{p_t, q_t\}$ 
(Proposition~\ref{prop:coupling}). This formulation restricts policy-gradient updates 
strictly to the subspace where student generations remain compatible with the guided policy, 
while delegating diverging states to the direct teacher-mode target $L_{\mathrm{TM}}$.
\end{remark}

\section{Additional Placement-Control Results}
\label{app:placement_control}

We compare two supervision-placement controls under the same exact-$q$
rollout. Random-TM matches the number of teacher-mode updates per
trajectory used by Ours while placing them at count-matched random
locations.

For TV-Weighted-TM, let $V_i$ denote the valid response positions of
trajectory $i$ and define
\[
m_i=\sum_{t\in V_i} C_{i,t},
\qquad
d_{i,t}=\TV(p_{i,t},q_{i,t}).
\]
We select $m_i$ valid positions without replacement using $d_{i,t}$ as
the sampling weight. Selected positions receive the same teacher-mode
loss as Ours, while the remaining valid positions retain RKL
supervision. The correction mask is used only to determine $m_i$ and
does not determine the selected locations. The exact-$q$ rollout and
maximal-coupling process are unchanged.

We implement the weighted sampling with exponential-race keys
$-\log U_{i,t}/d_{i,t}$ and select the $m_i$ smallest keys. If fewer
than $m_i$ valid positions have positive TV, the remaining positions
are filled uniformly.

\begin{table}[t]
\centering
\caption{
\textbf{Matched supervision-placement controls} for the 1.7B student on the same
seven-benchmark suite. Random-TM controls for the supervision budget,
while TV-Weighted-TM additionally favors positions with larger local
student--guided disagreement.
}
\label{tab:placement_control}
\begin{tabular}{lrr}
\toprule
Method & Mean@8 & Pass@8 \\
\midrule
TRB            & 27.90 & 44.60 \\
Random-TM      & 28.30 & 45.50 \\
TV-Weighted-TM & 28.24 & 46.77 \\
Ours           & \textbf{29.00} & \textbf{47.50} \\
\bottomrule
\end{tabular}
\end{table}

\begin{figure*}[t]
\centering
\includegraphics[width=0.98\textwidth]
{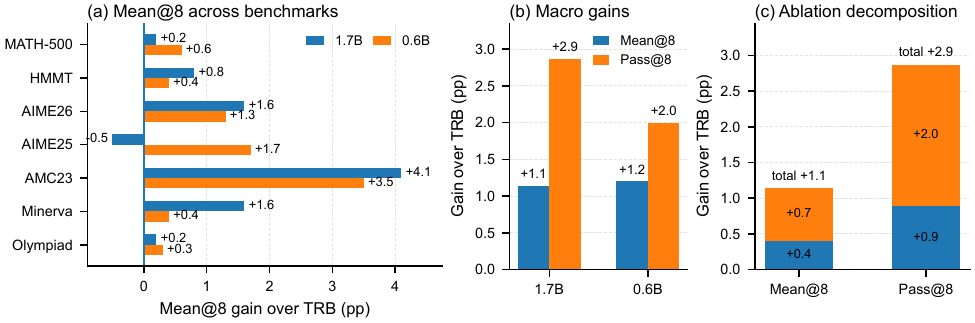}
\caption{
\textbf{Performance gains and supervision-placement ablation.}
(a) \textbf{Mean@8 gains over the matched TRB baseline across seven benchmarks}.
Our method improves six of seven benchmarks for the 1.7B student and all
seven for the 0.6B student.
(b) \textbf{Macro Mean@8 and Pass@8 improve at both student scales}.
(c) The \textbf{lower segment} shows the gain obtained by applying the same
number of teacher-mode updates at randomly selected positions
(TRB $\rightarrow$ Random-TM), while the \textbf{upper segment} shows the additional
gain from placing those updates at maximal-coupling corrections
(Random-TM $\rightarrow$ Ours).TV-Weighted-TM is reported separately in
Table~\ref{tab:placement_control}.
}
\label{fig:main_results_summary}
\end{figure*}

\section{Additional Fixed-Prefix Results}
\label{app:mechanism_details}

\paragraph{Matched fixed-prefix probe.}
We compare three matched 1.7B runs:
TRB, our correction-triggered teacher-mode objective, and a count-matched
random-placement control.
All three use the same Qwen3-1.7B-Base initialization, Qwen3-4B GRPO teacher,
exact-$q$ rollout, speculative block size $K=8$, annealing schedule,
training horizon, and seed.
TRB retains RKL supervision at all valid positions.
Our method replaces RKL with teacher-mode supervision at maximal-coupling
corrections.

Using the initial student with $\epsilon=0.02$, we construct a fixed probe
containing 1,024 genuine maximal-coupling correction positions and 1,024
matched accepted positions.
The fixed probe contains 2,048 positions spanning 213 prompts.
Prefixes, response positions, teacher distributions, and initial conflict
statistics are frozen throughout the analysis.

Let $v_T^*(h)$ denote the teacher's Top-1 token and let
$\mathcal{S}_{16}^{T}(h)$ denote its Top-16 token set.
For student distribution $S(\cdot\mid h)$, we measure
\[
C_1(h)=S(v_T^*(h)\mid h),
\qquad
C_{16}(h)=\sum_{v\in\mathcal{S}_{16}^{T}(h)}S(v\mid h).
\]
Thus, $C_{16}-C_1$ measures student probability on teacher-supported
Top-16 alternatives excluding the teacher argmax.

We first average metrics within each prompt and then compute paired
method differences, giving each prompt equal weight.
Confidence intervals use 10,000 paired bootstrap resamples clustered by
prompt ~\citep{efron1994bootstrap}.

To analyze heterogeneous placement effects, we rank fixed positions by their
initial $D_{\mathrm{KL}}(T\Vert S)$ and partition them into four equal-sized
conflict quartiles, from Q1 (lowest conflict) to Q4 (highest conflict).

The trust-region radius reaches exactly zero at step 51, after which
$q_t=p_t$ and no correction-triggered teacher-mode updates occur.

Table~\ref{tab:mechanism_persistence} reports the complete fixed-prefix
statistics underlying Figure~\ref{fig:mechanism_alignment}(a).
The reduction in $D_{\mathrm{KL}}(T\Vert S)$ persists after correction-triggered
teacher-mode supervision is disabled at step 51, while the accompanying
entropy reduction indicates that the student distribution becomes more
concentrated around teacher-supported mass.
\begin{table}[H]
\centering
\caption{
\textbf{Prompt-equal paired differences} between correction-triggered
teacher-mode supervision and the matched TRB control.
}
\label{tab:mechanism_persistence}
\begin{tabular}{lrrrr}
\toprule
Step
& $\Delta C_1$
& $\Delta C_{16}$
& $\Delta D_{\mathrm{KL}}(T\Vert S)$
& $\Delta H(S)$ \\
\midrule
50  & $+4.63$ pp & $+4.54$ pp & $-0.133$ & $-0.392$ \\
100 & $+4.30$ pp & $+2.69$ pp & $-0.154$ & $-0.249$ \\
200 & $+4.26$ pp & $+2.94$ pp & $-0.196$ & $-0.269$ \\
\bottomrule
\end{tabular}
\end{table}

\FloatBarrier
% ==================================================================
% C. TEACHER-MODE SUPERVISION
% ==================================================================

\section{Correction-Position Teacher Supervision}
\label{app:correction_supervision}

\paragraph{Teacher-mode supervision.}
At a correction prefix, let
\[
v^*=\operatorname*{arg\,max}_v T(v).
\]
Our teacher-mode loss is
\[
L_{\mathrm{TM}}
=
-\log p_\theta(v^*).
\]

Equivalently, because the teacher Top-1 distribution is the point mass
$\delta_{v^*}$,
\[
L_{\mathrm{TM}}
=
D_{\mathrm{KL}}(\delta_{v^*}\Vert p_\theta).
\]
Thus, teacher-mode supervision can be viewed as forward KL from a degenerate
teacher Top-1 target, while its implementation is simply hard-label
negative log-likelihood on the teacher argmax token.

\paragraph{Teacher-sampled supervision.}
We additionally consider a stochastic correction objective that keeps the
same maximal-coupling routing rule but samples the supervision target from
the teacher distribution. At a correction prefix, we draw
\[
\tilde v \sim T
\]
and optimize
\[
L_{\mathrm{TS}}
=
-\log p_\theta(\tilde v).
\]

In expectation over $\tilde v\sim T$,
\[
\mathbb{E}_{\tilde v\sim T}[L_{\mathrm{TS}}]
=
-\sum_v T(v)\log p_\theta(v),
\]
so this stochastic objective has the same expected gradient as
$D_{\mathrm{KL}}(T\Vert p_\theta)$.

\begin{table}[H]
\centering
\caption{
\textbf{Correction-supervision ablation} for the 1.7B student.
The coupling-aware variants use the same maximal-coupling routing rule and
trust-region schedule, differing only in the supervision target used at
correction positions.
}
\label{tab:correction_supervision}
\begin{tabular}{lrr}
\toprule
Correction supervision & Mean@8 & Pass@8 \\
\midrule
TRB & 27.90 & 44.60 \\
Teacher-sampled     & 28.38 & 45.94 \\
Teacher-mode (\textbf{SAKI}) & \textbf{29.00} & \textbf{47.50} \\
\bottomrule
\end{tabular}
\end{table}

Teacher-sampled supervision improves over the matched TRB baseline,
showing that coupling corrections are useful locations for direct
teacher-directed supervision beyond the particular choice of a hard
teacher-mode target. Teacher-mode supervision performs better in this
setting, improving over teacher-sampled supervision by 0.62 Mean@8 and
1.56 Pass@8 points. We therefore use teacher-mode supervision in the main
method.

% ==================================================================
% D. MAXIMAL COUPLING PROOF
% ==================================================================

\section{Proof of Proposition~\ref{prop:coupling}}
\label{app:maximal_coupling}

Let
\begin{equation}
Z\sim p
\end{equation}
be the student proposal.
Given $Z=v$, it is accepted with probability
\begin{equation}
a(v)
=
\min
\left(
1,
\frac{q(v)}{p(v)}
\right).
\end{equation}
Therefore,
\begin{equation}
\Pr(Z=v,C=0)
=
p(v)a(v)
=
\min\{p(v),q(v)\}.
\label{eq:accepted_mass}
\end{equation}
Summing over $v$,
\begin{equation}
\Pr(C=0)
=
\sum_v\min\{p(v),q(v)\}.
\end{equation}
Hence,
\begin{align}
\Pr(C=1)
&=
1-\sum_v\min\{p(v),q(v)\}
\\
&=
\frac{1}{2}
\sum_v
|p(v)-q(v)|
\\
&=
\TV(p,q).
\end{align}

For the rejected proposal,
\begin{align}
\Pr(Z=v,C=1)
&=
p(v)
\left[
1-
\min
\left(
1,
\frac{q(v)}{p(v)}
\right)
\right]
\\
&=
[p(v)-q(v)]_+.
\end{align}
Conditioning on $C=1$ gives
\begin{equation}
\Pr(Z=v\mid C=1)
=
\frac{
[p(v)-q(v)]_+
}{
\TV(p,q)
}.
\end{equation}

On rejection, the correction token is sampled from
\begin{equation}
r(v)
=
\frac{
[q(v)-p(v)]_+
}{
\sum_u[q(u)-p(u)]_+
}.
\end{equation}
Since
\begin{equation}
\sum_u[q(u)-p(u)]_+
=
\TV(p,q),
\end{equation}
we obtain
\begin{equation}
\Pr(Y=v\mid C=1)
=
\frac{
[q(v)-p(v)]_+
}{
\TV(p,q)
}.
\end{equation}

Finally, the unconditional probability of output token $v$ is
\begin{align}
\Pr(Y=v)
&=
\Pr(Y=v,C=0)
+
\Pr(Y=v,C=1)
\\
&=
\min\{p(v),q(v)\}
+
[q(v)-p(v)]_+
\\
&=
q(v).
\end{align}
Therefore, maximal coupling produces the exact $q$ marginal.

% ==================================================================
% E. BRIDGE INTERPRETATION
% ==================================================================

\subsection{Minimal-intervention property}
\label{app:minimal_intervention}

For any coupling $\gamma\in\Gamma(p,q)$,
\begin{align}
\Pr_\gamma(Z=Y)
&=
\sum_v
\Pr_\gamma(Z=v,Y=v)
\\
&\leq
\sum_v
\min\{p(v),q(v)\}.
\end{align}
Therefore,
\begin{align}
\Pr_\gamma(Z\neq Y)
&\geq
1-\sum_v\min\{p(v),q(v)\}
\\
&=
\TV(p,q).
\end{align}
Under the maximal-coupling construction in
Proposition~\ref{prop:coupling},
\begin{equation}
\Pr(Z=v,Y=v)
=
\min\{p(v),q(v)\},
\end{equation}
so equality is attained.
This proves Proposition~\ref{prop:minimal_intervention}.

The intervention-rate bound in
Corollary~\ref{cor:intervention_rate}
follows immediately from Pinsker's inequality
~\citep{cover2006elements},
\begin{equation}
\TV(p,q)
\leq
\sqrt{\frac{1}{2}D_{\mathrm{KL}}(q\Vert p)},
\end{equation}
together with the trust-region constraint.

\section{Teacher--Student Interpretation of the Residual Distribution}
\label{app:bridge_interpretation}

For the geometric bridge,
\begin{equation}
q_\beta(v)
=
\frac{
p(v)^{1-\beta}T(v)^\beta
}{
Z_\beta
},
\qquad
Z_\beta
=
\sum_u
p(u)^{1-\beta}T(u)^\beta.
\end{equation}
Assuming $\beta>0$ and positive softmax probabilities,
\begin{equation}
\frac{q_\beta(v)}{p(v)}
=
\frac{1}{Z_\beta}
\left(
\frac{T(v)}{p(v)}
\right)^\beta.
\end{equation}
Therefore,
\begin{align}
q_\beta(v)>p(v)
&\Longleftrightarrow
\frac{q_\beta(v)}{p(v)}>1
\\
&\Longleftrightarrow
\left(
\frac{T(v)}{p(v)}
\right)^\beta
>
Z_\beta
\\
&\Longleftrightarrow
\log\frac{T(v)}{p(v)}
>
\frac{\log Z_\beta}{\beta}.
\end{align}
Hence, the positive residual
\begin{equation}
[q_\beta(v)-p(v)]_+
\end{equation}
is supported only on tokens whose teacher-to-student likelihood ratio exceeds
a prefix-dependent threshold.

Symmetrically,
\begin{equation}
p(v)>q_\beta(v)
\end{equation}
corresponds to teacher-to-student likelihood ratios below this threshold.
Combined with Proposition~\ref{prop:coupling}, a correction can therefore be
viewed as replacing probability mass that is excessive under the student
relative to the bridge with probability mass favored by the bridge and
associated with sufficiently high teacher-to-student likelihood ratios.

% ==================================================================
% F. SPECULATIVE EXACTNESS
% ==================================================================

\section{Exactness of Engine-Resident Block Verification}
\label{app:spec_exact}

\begin{proposition}[Exactness of engine-resident block verification]
\label{prop:spec_exact}
Assume that the student and teacher distributions in a speculative block are
evaluated on the same proposal prefixes as sequential maximal coupling.
If proposals are committed only up to the first rejection, the correction is
sampled from the exact residual distribution, and all subsequent speculative
states are discarded, then the generated trajectory follows the same
autoregressive distribution as sequential sampling from $q$:
\begin{equation}
\Pr(y_{1:L}\mid x)
=
\prod_{t=1}^{L}
q_t(y_t\mid x,y_{<t}).
\end{equation}
\end{proposition}

\begin{proof}
Consider a speculative block beginning at a true prefix $h_t$.
The student autoregressively proposes
\begin{equation}
z_t,z_{t+1},\ldots,z_{t+K-1}.
\end{equation}
The teacher and bridge distributions are evaluated on the corresponding
proposal prefixes.

Suppose the first $j$ proposals are accepted.
Before any rejection occurs, the actual generated prefix is identical to the
proposal prefix.
Therefore, for each accepted position
\begin{equation}
t,t+1,\ldots,t+j-1,
\end{equation}
the precomputed student and teacher distributions are exactly the same
distributions that sequential maximal coupling would evaluate at the true
prefix.
Thus, committing these accepted tokens produces exactly the same conditional
transition as sequential coupling.

If the next proposal is rejected, the algorithm samples a correction from
\begin{equation}
r_{t+j}(v)
\propto
[q_{t+j}(v)-p_{t+j}(v)]_+,
\end{equation}
which is exactly the correction distribution used by sequential maximal
coupling at that prefix.
After this correction, the true prefix differs from the speculative proposal
prefix.
Consequently, all precomputed logits for later proposal positions are invalid
and are discarded.

The next speculative block is constructed from the new true prefix.
Therefore, after every committed token---whether accepted or corrected---the
algorithm has the same true prefix and uses the same next-token transition
kernel as sequential maximal coupling.

By induction over decoding steps,
\begin{equation}
Y_t
\sim
q_t(\cdot\mid x,Y_{<t})
\end{equation}
for every $t$, and hence
\begin{equation}
\Pr(y_{1:L}\mid x)
=
\prod_{t=1}^{L}
q_t(y_t\mid x,y_{<t}).
\end{equation}
Thus, speculative block verification changes only the execution schedule and
does not alter the target autoregressive distribution.
\end{proof}

% ==================================================================
% G. LIMITING CASES
% ==================================================================

\section{Limiting Cases and Sanity Checks}
\label{app:limits}

\subsection{$\epsilon=0$: Recovery of Standard Student Rollout}

If
\begin{equation}
\epsilon=0,
\end{equation}
the only feasible bridge distribution is
\begin{equation}
q_t=p_t.
\end{equation}
Therefore,
\begin{equation}
\TV(p_t,q_t)=0,
\end{equation}
and Proposition~\ref{prop:coupling} gives
\begin{equation}
C_t=0
\end{equation}
almost surely.

Consequently, the teacher-mode supervision term vanishes exactly:
\begin{equation}
\sum_t
M_tC_tL_{\mathrm{TM},t}
=
0.
\end{equation}
The objective in Eq.~\ref{eq:general_objective} reduces to
\begin{equation}
L
=
\frac{1}{N_{\mathrm{valid}}}
\sum_t
M_tL_{\mathrm{RKL},t},
\end{equation}
which is the ordinary student-rollout RKL objective.

\subsection{Correction Activity under Annealing}
\label{app:correction_activity}

Figure~\ref{fig:correction_schedule} compares the observed maximal-coupling
correction probability with the Pinsker upper bound induced by the current
trust-region radius.
As $\epsilon$ is annealed, correction activity decreases and vanishes at the
student-rollout endpoint.

\begin{figure}[t]
\centering
\includegraphics[width=\columnwidth]
{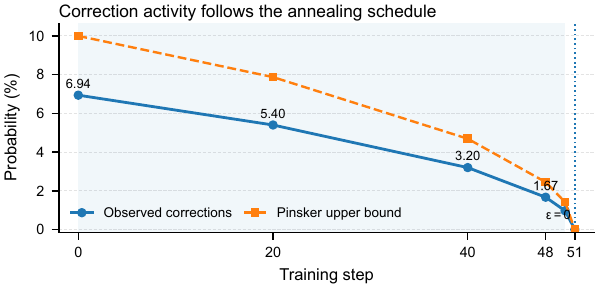}
\caption{
\textbf{Correction activity under annealing.}
The observed correction probability decreases with the trust-region radius,
remains below the Pinsker upper bound, and vanishes when $\epsilon$ reaches
zero at step 51.
}
\label{fig:correction_schedule}
\end{figure}

\FloatBarrier

\subsection{Teacher Endpoint}

If the teacher distribution itself satisfies
\begin{equation}
D_{\mathrm{KL}}(T_t\Vert p_t)\leq\epsilon,
\end{equation}
then the optimal trust-region bridge takes
\begin{equation}
\beta_t=1,
\qquad
q_t=T_t.
\end{equation}
Thus, the bridge smoothly interpolates between student rollout and the teacher
endpoint as the trust-region radius grows.

\section{Systems Implementation and Engineering Evolution}
\label{app:systems}

\subsection{Why Exact-$q$ Rollout Is Expensive}

Vanilla OPD admits a computationally convenient execution schedule: the student can first generate a complete response, after which the fixed
teacher scores the entire sequence in a batched forward pass.
Teacher-guided behavior blending instead requires the teacher distribution
during generation in order to construct $q_t$ at every visited prefix.
TRB explicitly notes this online-inference overhead
~\citep{plyusov2026trb}.

Ignoring transient activations, the generation-time resident memory can be
approximated as
\begin{equation}
M_{\mathrm{OPD}}(t)
\approx
W_s + KV_s(n_t)
\end{equation}
for student-only generation, whereas online teacher-guided generation requires
\begin{equation}
M_q(t)
\approx
W_s + W_T + KV_s(n_t) + KV_T(n_t).
\end{equation}
The additional generation-time footprint is therefore approximately
\begin{equation}
\Delta M_{\mathrm{gen}}(t)
\approx
W_T + KV_T(n_t).
\end{equation}

The arithmetic required by the geometric bridge itself is comparatively simple
vector computation.
In practice, the dominant difficulty in our original implementation was the
execution structure surrounding this computation: token-level RPCs, transfer of
full-vocabulary logits, host synchronization, repeated prefix handling, and
Python control flow.

\subsection{From Token-Wise to External Block Verification}

Our first correctness-oriented exact-$q$ implementation operated one token at a
time.
For every decoding position, it independently:
(i) queried the student for a one-token proposal and full-vocabulary
distribution,
(ii) queried the teacher at the same prefix,
(iii) transferred both distributions outside the inference servers,
(iv) constructed the trust-region bridge in Python,
(v) executed maximal coupling, and
(vi) appended the resulting token before repeating the process.

For a response of length $L$, the execution cost was therefore approximately
\begin{equation}
T_{K=1}
\approx
L
\left(
T_{\mathrm{student\ RPC}}
+
T_{\mathrm{teacher\ RPC}}
+
T_{\mathrm{dense\ transfer}}
+
T_{\mathrm{bridge}}
+
T_{\mathrm{coupling}}
+
T_{\mathrm{host}}
\right).
\end{equation}

In an early unoptimized prototype ($K=1$), external per-token RPC and logit serialization yielded only $\sim$5 tokens/s (compared to $\sim$3.02K tokens/s for student-only generation under the same pilot), illustrating the severe communication bottleneck of naive online bridge evaluation.

We next introduced external $K$-token proposals.
For $K=8$, one student request proposes a block and one teacher suffix-scoring
request evaluates the corresponding proposal prefixes.
This reduces model-service calls from approximately $O(L)$ to approximately
$O(L/K)$ when blocks are largely accepted.
However, bridge construction, dense-logit transfer, coupling, metadata
assembly, and commit decisions still occur in the external Python coordinator.

\subsection{Engine-Resident $q$ Backend}

The final backend moves the entire recurrent block loop inside the inference
runtime.
Each engine replica contains both a student draft model and the fixed teacher.
For each speculative wave it performs:
\begin{enumerate}
    \item student sampling of up to $K$ proposal tokens;
    \item teacher verification of the proposal path;
    \item batched full-vocabulary construction of $q_t$;
    \item fixed-wave batched solution of $\beta_t$;
    \item vectorized maximal-coupling acceptance;
    \item batched positive-residual correction sampling;
    \item first-rejection KV commit/rollback; and
    \item compact token-aligned metadata packing.
\end{enumerate}

Student and teacher logits remain inside the engine.
For a batch of proposal blocks, they are represented as tensors of shape
$[B,K,V]$ and reshaped to $[BK,V]$ for bridge construction.
The trust-region solve uses 16 fixed bisection waves.
Rows that have converged are masked on GPU rather than triggering
host-synchronized early termination.

For a proposal token $a$, maximal-coupling acceptance is evaluated as
\begin{equation}
\alpha(a)
=
\exp
\left(
\min
\left\{
0,
\log q(a)-\log p_s(a)
\right\}
\right).
\end{equation}
A prefix cumulative product over the acceptance indicators identifies the
longest consecutively accepted proposal prefix.
If the first rejection occurs at position $j$, only proposals before $j$ and
the sampled correction are committed; every later speculative state is
discarded because it was computed under an invalid prefix.

The external training loop receives only compact metadata required by the
objective, rather than full-vocabulary logits.

\subsection{Model and Cache Lifecycle}

In our 8-GPU training layout, GPUs 0--3 execute the student actor update and
GPUs 4--7 host $q$-engine replicas.
Each $q$-engine keeps a fixed teacher and a draft replica of the evolving
student resident in memory.

After every actor update, only the student draft weights are synchronized.
The teacher remains fixed throughout training.
Cached states associated with the previous student version are invalidated
before new rollout requests resume.
This avoids repeatedly redeploying or synchronizing the unchanged teacher.

The inference engines also use a separate cache lifecycle from the actor
training path.
Keeping the $q$-engine model and cache pools resident avoids repeated
release/reinitialization overhead and ensures that speculative proposal and
verification states remain under the engine's ownership.

Full-vocabulary distributions are consumed entirely inside the $q$-engine.
The response returned to the training loop contains only token-aligned metadata,
including proposal and committed tokens, the correction mask, selected
student/$q$ log-probabilities, $\beta$, bridge KL, residual mass, acceptance
probability, and the effective trust-region radius.

\begin{table*}[t]
\centering
\small
\caption{
\textbf{Engineering progression of the exact-$q$ rollout implementation}.
The first two rows reflect early exploratory setups, while the final two rows
benchmark mature implementations under the matched workload.
}
\label{tab:engineering_evolution}
\begin{tabular}{lrrl}
\toprule
Implementation & $K$ & Tok/s & Execution \\
\midrule
Token-wise prototype      & 1 & $\sim$5   & Per-token external RPC \\
Early external block path & 8 & 165--299   & External block verification \\
Final external loop       & 8 & 776        & Mature Python/RPC pipeline \\
Engine-resident backend   & 8 & 3,276      & GPU-resident bridge/coupling \\
\bottomrule
\end{tabular}
\end{table*}

\begin{table}[t]
\centering
\caption{\textbf{Representative integrated training-step timing}.}
\label{tab:step_timing}
\begin{tabular}{lr}
\toprule
Component & Time (s) \\
\midrule
$q$ generation & 189.5 \\
Actor update & 24.9 \\
Student weight sync & 2.5 \\
Whole step & 232.6 \\
\bottomrule
\end{tabular}
\end{table}

Student synchronization accounts for only approximately 1.1\% of the
representative step wall-clock, indicating that online rollout rather than
draft-weight synchronization remains the dominant systems cost.

\subsection{Correctness Validation and Limitations}

All system optimizations are required to preserve the same mathematical
sampling semantics as the reference implementation.
We validate the optimized backend along four dimensions.

\paragraph{Bridge correctness.}
GPU bridge construction is compared against a CPU reference on randomized
small-vocabulary distributions.
We verify normalization, the trust-region constraint, the exact
$\epsilon=0$ student endpoint, and the $\beta=1$ teacher endpoint when feasible.

\paragraph{Coupling correctness.}
Monte Carlo samples from maximal coupling are compared against direct samples
from $q$.
We additionally test no-rejection blocks, first-rejection commit semantics,
residual correction sampling, EOS handling in discarded speculative suffixes,
and token-aligned metadata.

\paragraph{Training-state correctness.}
We verify that student weight synchronization changes the draft-model checksum
while the teacher checksum remains fixed, and that cached states from old
student versions are invalidated.

\paragraph{Routing correctness.}
Validation always generates from the current student alone.
When the effective trust-region radius is exactly zero, the rollout follows the
student-only fast path, teacher $q$-forward count is zero, and the correction
mask is identically zero.

The optimized system does not eliminate the intrinsic cost of online teacher
verification.
Its throughput therefore remains below student-only generation and depends on
response length, correction rate, and the number of committed tokens per
verification wave.

% ==================================================================
% H. OBJECTIVE / IMPLEMENTATION SEMANTICS
% ==================================================================

\section{Additional Objective and Implementation Details}
\label{app:implementation}

\paragraph{Correction masking.}
At correction positions, we set the RKL contribution exactly to zero.
The correction loss does not stack with RKL.

\paragraph{Normalization.}
Both accepted-position RKL terms and correction terms are normalized by the
total number of valid response tokens,
\begin{equation}
N_{\mathrm{valid}}
=
\sum_t M_t.
\end{equation}

Accepted-position RKL terms and correction-position teacher-mode terms receive
unit weight and are jointly normalized by the total number of valid response
tokens.
No additional correction-loss coefficient is used.

\paragraph{Role of the residual correction token.}
The correction token sampled from the residual distribution becomes the actual
response token at that position and therefore determines all subsequent
prefixes.
However, it is not used as the direct supervision target.
At correction positions, the direct training target is the teacher's
highest-probability token and therefore need not equal the sampled residual
token.

\paragraph{Full-vocabulary rollout.}
The intermediate rollout distribution $q_t$ and maximal-coupling computation
operate over the full vocabulary.
The teacher Top-1 target is used only for the correction supervision
objective and does not truncate or otherwise modify the rollout distribution.

\paragraph{Validation.}
Evaluation should always generate from the current student alone, rather than
from the teacher-guided rollout engine, so that reported benchmark performance
measures the distilled student itself.

\end{document}